\documentclass[10pt,twocolumn,letterpaper,abstract=true]{scrartcl}

\usepackage[table]{xcolor}

\usepackage[utf8]{inputenc}             
\usepackage[english]{babel} 

\usepackage{amsmath}
\usepackage{amssymb}
\usepackage{amsthm}
\usepackage{amsbsy} 
\usepackage{amsfonts}

\usepackage{algorithmic}
\usepackage{algorithm}
\usepackage{array}
\usepackage[caption=false,labelfont=sf,textfont=sf]{subfig}
\usepackage{textcomp}
\usepackage{stfloats}
\usepackage{url}
\usepackage{verbatim}
\usepackage{graphicx}

\usepackage{mathtools}
\usepackage{booktabs}
\usepackage{faktor}
\usepackage{quiver}
\usepackage{csvsimple}
\usepackage{soul}

\usepackage{lmodern}
\usepackage{fourier}

\usepackage{mathtools}
\usepackage{tabularx}
\usepackage{enumitem}					
\usepackage{tikz}
\usepackage{tikz-cd}
\usetikzlibrary{calc}

\definecolor{parula1}{rgb} {0.0000, 0.4470, 0.7410 } 
\definecolor{parula2}{rgb} {0.8500, 0.3250, 0.0980 } 
\definecolor{parula3}{rgb} {0.9290, 0.6940, 0.1250 } 
\definecolor{parula4}{rgb} {0.4940, 0.1840, 0.5560 } 
\definecolor{parula5}{rgb} {0.4660, 0.6740, 0.1880 } 
\definecolor{parula6}{rgb} {0.3010, 0.7450, 0.9330 } 
\definecolor{parula7}{rgb} {0.6350, 0.0780, 0.1840 } 

\newcommand{\tr}    {\intercal}
\newcommand{\kron}  {\otimes}
\newcommand{\RR}     {\mathbb{R}}
\renewcommand{\P} 	{\mathbb P} 	
 	
\newcommand{\mF}    {\mathcal{F}}
\newcommand{\mP}    {\mathcal{P}} 
\newcommand{\vx} 	{\mathbf{x}}

\theoremstyle{plain}
\newtheorem{theorem}{Theorem}
\newtheorem{corollary}{Corollary}
\newtheorem{prop}{Proposition}
\newtheorem{lemma}{Lemma}
\theoremstyle{definition}
\newtheorem{definition}{Definition}
\theoremstyle{remark}
\newtheorem{remark}{Remark}

\DeclareMathOperator{\veco} {vec}
\DeclareMathOperator{\vech} {vech}

\DeclareMathOperator{\rank} {rank}

\DeclareMathOperator{\jac}  {\mathsf{D}}

\usepackage[style = numeric-comp, sorting = none]{biblatex}

\definecolor{iccvblue}{rgb}{0.21,0.49,0.74}
\usepackage[breaklinks,colorlinks,allcolors=iccvblue]{hyperref}

\usepackage[margin=2cm]{geometry}

\bibliography{Definitions,NostriNEW,AndreaNEW,FedeNEW}

\setkomafont{author}{\large}

\title{\LARGE An Algebraic Geometry Approach to\\ Viewing Graph Solvability}

\author{Federica Arrigoni, Kathl\'en Kohn, Andrea Fusiello, Tomas Pajdla
}

\begin{document}
\maketitle

\begin{abstract}
 The concept of viewing graph solvability has gained significant interest in the context of structure-from-motion. A viewing graph is a mathematical structure where nodes are associated to cameras and edges represent the epipolar geometry connecting overlapping views. Solvability studies under which conditions the cameras are uniquely determined by the graph. In this paper we propose a novel framework for analyzing solvability problems based on Algebraic Geometry, demonstrating its potential in understanding structure-from-motion graphs and proving a conjecture that was previously proposed.
 \end{abstract}

\section{Introduction}
\label{sec:intro}

In recent years, there has been a notable increase in interest surrounding the concept of viewing graph solvability in the field of Computer Vision \cite{LeviWerman03,RudiPizzoliAl11,TragerHebertAl15,TragerOssermanAl18,ArrigoniFusielloAl21,ArrigoniPajdlaAl23}. This concept plays a pivotal role in the domain of structure-from-motion (SfM) \cite{CrandallOwensAl11,OzyesilVoroninskiAl17,ChatterjeeGovindu17,SarlinLindenbergerAl23,ManamGovindu23}, which aims to reconstruct three-dimensional scenes from a multitude of images. A \emph{viewing graph} \cite{LeviWerman03} is a mathematical structure in which the nodes represent the cameras that capture the scene and the edges connect the cameras that have overlapping views. 
More precisely, an edge is present between two nodes if and only if it is possible to estimate
the geometric relationship between the two cameras, encoded in the fundamental matrix (assuming an uncalibrated scenario).
This defines a constraint system that is {classically} considered \emph{solvable} if the information encoded in the fundamental matrices uniquely determines all cameras in the scene, up to a global projective transformation (see Figure \ref{fig:problem_solvability}). 
Despite significant advances have been recently made both from the theoretical and practical point of view \cite{ArrigoniFusielloAl21,ArrigoniPajdlaAl23}, viewing graph solvability
still presents open issues, as discussed in the next subsection.

\subsection{Related Work}
\label{sec:related_work}

It is well known that a \emph{single} fundamental matrix uniquely determines the two perspective cameras up to a projective transformation \cite{HartleyZisserman04}. However, when considering \emph{multiple} fundamental matrices attached to the edges of a viewing graph, there may be cases with many solutions or no solution at all. 

A viewing graph is called \emph{solvable}  if, 
for almost all choices of cameras, there are no other sets of cameras yielding the same fundamental matrices (up to global projective transformation). In other terms, 
it is assumed that a solution exists (i.e., a set of cameras compliant with the given fundamental matrices), and the question is whether such solution is the only one or there are more.
The concept of \emph{solving} viewing graph (later called \emph{solvable} by \cite{TragerOssermanAl18}) was first introduced in \cite{LeviWerman03} where small incomplete graphs (up to six cameras) were manually analyzed, by reasoning in terms of how to uniquely recover the missing fundamental matrices from the available ones. 

The authors of \cite{LeviWerman03} also derived a \emph{necessary condition} for solvability, namely the property that all the nodes have degree at least two and no two adjacent nodes have degree two. Later, additional necessary conditions were developed: a solvable graph must be biconnected \cite{TragerOssermanAl18}; it must have at least $(11n-15)/7$ edges, with $n$ being the number of nodes \cite{TragerOssermanAl18}; it must be bearing rigid \cite{ArrigoniFusielloAl22}. The latter means that, as expected, a graph that is solvable with unknown intrinsic parameters is also solvable when they are known \cite{ArrigoniFusiello18,TronCarloneAl15,KarimianTron17}. 

\emph{Sufficient conditions} are also available: in \cite{TragerHebertAl15} it is proved that those graphs which are constructed from a 3-cycle by adding nodes of degree 2 one at a time are solvable; \cite{TragerOssermanAl18} introduces specific ``moves'' which can be applied to a graph, possibly transforming it into a complete one, in which case the graph is solvable.

\begin{figure}[t]
\centering
\usetikzlibrary{fit,shapes.arrows}
\begin{tikzpicture}[scale=1.5, every node/.style={circle, draw, font=\footnotesize\sffamily}]
  \node[above, draw=none,] at (1, 2) {\textbf{Viewing Graph}};
  \node (P1) at (0, 0) {$P_1$};
  \node (P2) at (0, 2) {$P_2$};
  \node (P3) at (2, 2) {$P_3$};
  \node (P4) at (2, 0) {$P_4$};
  \node (P5) at (1, 1) {$P_5$};
  \draw[-] (P1) -- (P2) node[draw=none,midway,above left] {$F_{12}$};
  \draw[-] (P2) -- (P3) node[draw=none,midway,above] {$F_{23}$};
  \draw[-] (P3) -- (P4) node[draw=none,midway,below right] {$F_{34}$};
  \draw[-] (P4) -- (P1) node[draw=none,midway,above] {$F_{41}$};
  \draw[-] (P1) -- (P5) node[draw=none,midway,above left] {$F_{15}$};
  \draw[-] (P3) -- (P5) node[draw=none,midway,above left] {$F_{53}$};
 \node[anchor=west,draw=none] at (3.8, 1) {$P_1, P_2, P_3, P_4, P_5$};
 \node[above, draw=none,] at (4.5, 1) {\textbf{Cameras}};
 \coordinate[] (a) at (2.7, 1); 
 \coordinate[] (b) at (3.6, 1);
 \node[single arrow,  draw=black, fill=none, 
      minimum width = 10pt, single arrow head extend=3pt,
      minimum height=10mm, fit= (a) (b)] {}; 
      \node[above, draw=none] at ($(a)!0.5!(b)$) {?}; 
\end{tikzpicture}
    \caption{The solvability problem considers the following theoretical question: given a set of fundamental matrices encoded in a graph, how many camera configurations are compliant with such fundamental matrices?
    }
    \label{fig:problem_solvability}
\end{figure}
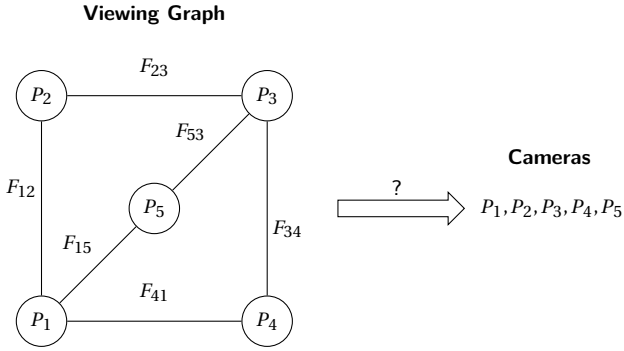

In fact, necessary or sufficient conditions alone are not enough to classify all possible cases so a \emph{characterization} is required. In this respect, the authors of \cite{TragerOssermanAl18,ArrigoniFusielloAl21} study solvability using principles from Algebraic Geometry.
Specifically, a polynomial system of equations was derived in \cite{TragerOssermanAl18}, so that solvability can be tested by counting the number of solutions of the system with algebraic geometry tools (e.g., Gr\"obner basis computation). Building on \cite{TragerOssermanAl18}, the authors of \cite{ArrigoniFusielloAl21} improve efficiency by deriving a simplified polynomial system with fewer unknowns. Still, the largest example tested in \cite{ArrigoniFusielloAl21} is a graph with 90 nodes, which is far from the size of structure-from-motion datasets appearing in practice. 

The main drawback of \cite{TragerOssermanAl18,ArrigoniFusielloAl21} is that solving polynomial equations is computationally highly demanding, therefore limiting the practical usage of this characterization of solvability. For this reason, the related notion of \emph{finite solvability} has been explored \cite{TragerOssermanAl18}. Specifically, a graph is called
\emph{finitely solvable} if, 
for almost all choices of cameras, there is a \emph{finite} set of cameras that gives the same fundamental matrices (up to global projective transformation).
This concept represents a proxy for (unique) solvability since it does not exclude the presence of more than one solution (e.g., two distinct solutions); however, it has been shown to be more practical since it can be deduced from the rank of a suitable matrix. Later, the authors of \cite{ArrigoniPajdlaAl23} improved the efficiency of this formulation and
developed a method to partition an unsolvable graph into maximal components that are finitely solvable. 

Problems related to solvability, which are not addressed in this paper, include the \emph{compatibility} of fundamental matrices, namely, whether a camera configuration exists that  produce the given fundamental matrices \cite{SenguptaAmirAl17,BratelundRydell23}, and the practical task of retrieving cameras from fundamental matrices \cite{SinhaPollefeysAl04,KastenGeifmanAl19,ColomboFanfani21,MadhavanFusielloAl24}.

\subsection{Contribution}

In the wake of the emerging field of Algebraic Vision \cite{KileelKohn23}, in this work we advance the understanding of viewing graphs by focusing on the notion of \emph{finite solvability}.
The main contributions can be summarized as follows: 

\begin{itemize}
    
    \item We derive a new formulation of the problem that is more direct (hence more \emph{intuitive}) than previous work, as our equations explicitly involve cameras and fundamental matrices. Previous sets of equations \cite{TragerOssermanAl18,ArrigoniPajdlaAl23}, instead, are harder to interpret, as they involve unknown projective transformations representing the problem ambiguities.

    \item We show that, by evaluating  the rank of the \emph{Jacobian matrix} of our polynomial equations in a fabricated solution, we can
    test finite solvability. 
    It is not immediate that this Jacobian check can assess the presence of a \emph{finite} number of solutions overall, being designed as a local analysis.
    Our proof, based on the Fiber Dimension Theorem \cite{Shafarevich13}, confirms a conjecture made in our preliminary work \cite{ArrigoniFusielloAl24}.

    \item Our method for testing finite solvability naturally extends to an algorithm for \emph{graph partitioning} into the maximal components that are finite solvable, to be applied to unsolvable cases with infinitely many solutions. The number of  unknowns depends on the number of nodes in the graph, that are typically much inferior to the number of edges used by previous work \cite{TragerOssermanAl18,ArrigoniPajdlaAl23}. This permits us to set the state of the art in terms of \emph{efficiency} on large graphs coming from SfM datasets. 
\end{itemize}

This paper is an extended version of our preliminary study \cite{ArrigoniFusielloAl24}. The manuscript is organized as follows. Section \ref{sec_theory} reviews relevant background on solvability and finite solvability. Section \ref{sec_equations} presents our theoretical contributions and introduces the set of polynomial equations employed in our formulation. Section \ref{sec_method} details our approach for testing finite solvability and extracting maximal components. Section~\ref{sec_implementation} reports formulas for derivatives, useful both for theory and practice. Experiments on synthetic and real viewing graphs are reported in Section \ref{sec_experiments}, while the conclusion is drawn in Section \ref{sec_conclusion}.

\section{Background}
\label{sec_theory}

Let $P_1, \dots, P_n $ denote $n$ uncalibrated cameras, represented by ${3 \times 4}$ full-rank matrices up to scaling, identified with elements of $\P^{11}$.
Let ${G}=({V},{E})$ be an undirected graph with node set ${V}=\{1, \dots, n\}$ and edge set ${E} \subseteq \{1, \dots, n\} \times \{1, \dots, n\}$ representing a \emph{viewing graph} of an uncalibrated structure-from-motion problem. We denote the cardinality of the vertex set with $n=|V|$, the number of edges with 
$m=|E|$ and the fundamental matrix of  $(i,j) \in {E} $ with $F_{ij}$. 
We use the following terminology: given a graph $G = (V,E)$, a \emph{configuration} is a map $\mP: V \to \P^{11}\text{ (full rank)}$ that assigns nodes to cameras. 
A \emph{framework} is a pair $(G,\mP)$, where $G
= (V, E)$ is a graph and $\mP$ is a configuration. 

Fundamental matrices are {equivalence classes of rank two} $3 \times 3$ matrices up to {non-zero} scaling, identified with elements of $\P^8$. They are assigned to edges via the map $ \mF_G(\mP) = [\ldots \, \mF({P_i,P_j}) \ldots]$  where $\mF({P_i,P_j})$ evaluates the fundamental matrix $F_{ij}$ on edge $(i,j) \in E$.
 Hence:
\begin{equation} \label{eq:fundamentalMatrixMap}
\mF:  \P^{11} \times \P^{11} \to \P^{8},  \quad   (P_i, P_j) \mapsto  F_{ij},
\end{equation}
where $F_{ij}$ is the fundamental matrix defined by  cameras $P_j$ and $P_i$. 
One way of specifying this map entry-wise is:
\begin{equation}
 [F_{ij}]_{h,k} = (-1)^{h+k} \det 
 \begin{bmatrix}
 P_i^k \\
 P_j^h
 \end{bmatrix},
\label{eq:Fuv_det}
\end{equation}
where $P_i^k$ denotes the $2 \times 4$ sub-matrix of camera $P_i$ obtained by removing row $k$ (and similarly for $P_j^h$ with row $h$). {Note that Eq.~\eqref{eq:Fuv_det} gives the zero matrix if $P_i$ and $P_j$ have coincident centres, meaning that the map $\mathcal{F}$ is undefined in the projective sense: in this scenario, it is known that the fundamental matrix is not uniquely defined \cite{HartleyZisserman04}. Therefore, we assume henceforth  that cameras have distinct centres.}

In particular, since \eqref{eq:Fuv_det} is a polynomial, we see that this is 
an \emph{algebraic map}, i.e.,  a function between algebraic varieties given locally by rational functions. 
For us, an \emph{algebraic variety} is the solution set of a system of polynomial equations. 

The \textbf{key question} is the following: 
\begin{quote}
    \emph{given a framework $(G,\mP_0)$, how many configurations $\mP$  exist yielding the same fundamental matrices?}
\end{quote}
In algebraic terms, we want to study the {cardinality} of the fibers (i.e., pre-images of points) of $\mF_G$; hence the question can be rephrased as: 
\begin{quote}
\emph{
what is the {cardinality} of $\mF_G^{-1}(\mF_G(\mP))$}?
\end{quote}
%
In formulating this question, we identify all configurations that are projectively equivalent. For instance, if we state that a configuration is unique, this is always intended up to a global projective transformation, which is an element of $\mathrm{PGL}_4$, the Projective General Linear Group on $\mathbb{P}^3$.

\begin{definition}[Solvable framework \cite{TragerOssermanAl18}]
Let $\mP = \{ P_1, \dots, P_n \}$ be a configuration of cameras, and let ${G}$ be a graph. The framework $({G},\mP)$ is called \emph{solvable} 
if all camera configurations yielding the same fundamental matrices as $\mP$ are obtained from $\mP$ via a global projective transformation.
In other words, $\mF_G^{-1}(\mF_G(\mP))$ is a single point, modulo $\mathrm{PGL}_4$.
\end{definition}

Studying the solvability of frameworks requires considering the actual camera configuration and accounting for special cases, such as collinear centers. To avoid this, a generic configuration is typically considered, leading to another concept of solvability, which is a property of the graph itself.

\begin{definition}[Solvable graph\cite{TragerOssermanAl18}]
A graph ${G}$ is called \emph{solvable} if it is solvable for a \emph{generic} configuration of cameras. In other words, ${G}$ is solvable if and only if, generically, the non-empty fibers of $\mF_G$ are points, modulo $\mathrm{PGL}_4$. 
\label{def_solvable}
\end{definition}

Solvability, in this context, does not concern finding a specific solution (i.e., a camera configuration producing given fundamental matrices). Rather, it focuses on
counting the solutions,  assuming at least one solution exists.
This interpretation is consistent with prior research \cite{TragerOssermanAl18, ArrigoniFusielloAl22}.
%

Determining the solvability of a graph requires solving a polynomial system of equations \cite{TragerOssermanAl18, ArrigoniFusielloAl22}. This process is computationally demanding, rendering it prohibitive for large or dense graphs often encountered in practice.
%
A relaxed notion is \emph{finite solvability},  requiring a finite number of solutions, as opposed to one solution. 

\begin{definition}[Finite solvable graph\cite{TragerOssermanAl18}]
A graph ${G}$ is called \emph{finite solvable} if and only if,
 generically, 
 the non-empty fibers of  $\mF_G$ are finite, modulo $\mathrm{PGL}_4$. 
\end{definition}

\begin{remark}
{Checking the finite solvability of graphs is computationally more feasible because it can be checked locally.}
In the  formulation of \cite{TragerOssermanAl18}, polynomial equations are derived by reasoning on the problem ambiguities, whose solution set forms a smooth 
algebraic variety endowed with a group structure.  
This property implies that the dimension of this variety  coincides with the dimension of its tangent space at the identity, which can be computed efficiently. Specifically, since the tangent space is a linear space, its dimension reduces to the rank of a linear system of equations. This dimension reveals whether the original polynomial system admits a finite number of solutions or, equivalently, whether the graph is finitely solvable.
\end{remark}

Our approach targets the same notion of finite solvability but with a different polynomial system.
In contrast to \cite{TragerOssermanAl18} (later improved by \cite{ArrigoniPajdlaAl23}), our equations directly involve cameras and fundamental matrices, thereby gaining efficiency and interpretability by design. However, this comes at the price of loosing the group property (cameras are not invertible matrices), thus requiring a different mathematical approach.

\begin{remark}
 Although finite solvability is only a necessary condition for solvability, it remains a valuable property. It can be interpreted as a \emph{local solvability}, meaning that  the solution is unique within a neighborhood of the given configuration.      
\end{remark}

\section{Theoretical Results}
\label{sec_equations}

This section is devoted to our theoretical results, which set the basis for the proposed method for checking finite solvability. We first prove a new characterization of the problem and then detail our choice of polynomial equations.

\subsection{Characterization of Finite Solvability}

Our goal is to study the \emph{generic} finite solvability of a graph $G$, i.e.,  we ask whether 
$\mF_G^{-1}(\mF_G(\mP))$ for \emph{generic} $\mP$ is a finite set or an infinite one (modulo $\mathrm{PGL}_4$), which is equivalent to studying the dimension of $\mF_G^{-1}(\mF_G(\mP))$.
The dimension of an algebraic variety intuitively quantifies the number of independent parameters required to describe points on the variety, much like the dimension of a vector space or a manifold. While we omit a formal definition here, we note that a variety has dimension 0 if and only if it consists of a finite number of points \cite[Chap.\ 1]{Shafarevich13}. 

To be more concrete, one can assign random cameras $\mP_0$ to nodes of $G$ and compute the fundamental matrices using $\mF_G(\mP_0)$ with Equations \eqref{eq:fundamentalMatrixMap} and \eqref{eq:Fuv_det}. The task is then to determine how many camera configurations produce the same fundamental matrices as $\mP_0$, modulo $\mathrm{PGL}_4$.
%
This is a \emph{global} question, but it can be addressed through a \emph{local} analysis, made possible by  Proposition~\ref{lemma:Jcheck->solvable} that we are going to prove at the end of the section, after recalling some results from Algebraic Geometry. 

\begin{lemma}[Fiber Dimension Theorem]
\label{lemma:constant_fiber_dim}
If $f: X \to Y$ is an algebraic map between irreducible varieties (over $\mathbb{C}$), then 
\begin{equation}
   \dim X = \dim f^{-1}(f(x)) + \dim \mathrm{im}(f) 
\end{equation}
for almost all $x \in X$, where $\mathrm{im}$ denotes the image of the map.
\end{lemma}

An \emph{irreducible variety} is a variety that cannot be written as the union of two non-empty proper sub-varieties. 
Lemma \ref{lemma:constant_fiber_dim} is known as the Fiber Dimension Theorem \cite[Chap.\ 1.6.3]{Shafarevich13}. 
It establishes that the fiber dimension is constant on generic points, and that this dimension is dual to the dimension of the image parameterized by the map. This relation extends the rank-nullity theorem from Linear Algebra to polynomial maps. In particular, it says that an algebraic map \( f \) either has (generically) finite fibers or it has generically infinite fibers. In other words, all generic fibers have the same dimension, hence the behavior of a single fiber is enough to get global information.

Another standard fact in algebraic geometry is that, at a generic point in the domain of an algebraic map,
the rank of the Jacobian matrix equals the dimension of the image \cite{Shafarevich13}.

\begin{lemma}[Lemma 2.4 in Chap.~2.6 of\cite{Shafarevich13}]
\label{lemma:Shafa2.4}
If $f: X \to Y$ is an algebraic map between irreducible varieties, then, for almost all $x \in X$,
\begin{equation}
\dim \mathrm{im}(f) = \rank \jac f(x). 
\end{equation}
\end{lemma}

From these two lemmas if follows immediately that:
\begin{corollary}
If $f: X \to Y$ is an algebraic map between irreducible varieties (over $\mathbb{C}$), then 
\begin{equation}
  \dim f^{-1}(f(x)) = \dim X - \rank \jac f(x)   
\end{equation}
    for almost all $x \in X$.
    \label{lemma:jacobian_fiber}
\end{corollary}

\smallskip

We are now able to characterize the finite solvability of a graph \( G \) in terms of the rank of the Jacobian matrix associated with \(\mF_G\), the function that computes the fundamental matrices along the edges of \( G \).

In the following we are going to represent matrices in an affine chart, and consequently work with the affine version\footnote{\label{footnoteAffineChart}One way of fixing an affine chart in the domain $\P^{11} \times \P^{11}$ is by setting one of the 12 entries in each camera matrix to 1. Similarly, in the codomain $\P^8$, we can choose an affine chart by fixing one of the entries of the fundamental matrix to be 1, e.g., the very last entry. That way, the map $\mF^\mathrm{aff}$ becomes a rational map: Its coordinate functions are fractions $\frac{[F_{ij}]_{h,k}}{[F_{33}]_{h,k}}$ of the polynomials in \eqref{eq:Fuv_det}.} of the map $\mF$, denoted by $\mF^\mathrm{aff}: \RR^{11} \times \mathbb{R}^{11} \to \mathbb{R}^8$.
 With a little abuse of notation we are not going to distinguish between  a projective element and its affine representation, as the map where they appear will be enough to  disambiguate.

\begin{prop}
\label{lemma:Jcheck->solvable}
Let  $\mF_G^\mathrm{aff}: (\RR^{11})^{|V|} \to (\RR^8)^{|E|}$ be the affine version of the algebraic map that computes the fundamental matrices along a viewing graph $G$ with nodes $V$ and edges $E$. It 
has a Jacobian $ \jac{\mF_G^\mathrm{aff}} $ made of blocks of  size $ 8 \times 22$,  defined as:
\begin{equation}
[\jac{\mF_G^\mathrm{aff}}]_{i,j} \coloneqq  \frac{\partial  \mF^\mathrm{aff}}{\partial P_i, P_j}.
\end{equation}
Then, for a generic configuration $\mP_0$, we have:
$$ \mathrm{rank} (\jac{\mF_G^\mathrm{aff}} (\mP_0)) = 11|V| - 15    \iff   G \text{ is finite solvable} .$$
\end{prop}

\begin{proof}
The domain of our map is the  set of camera matrices $(\RR^{11})^{|V|}$  (interpreted as $3 \times 4$ matrices). 
The map $\mF_G^\mathrm{aff}$ is well-defined on generic cameras in this domain. Since both domain $X=(\RR^{11})^{|V|}$ and codomain $Y=(\RR^{8})^{|E|}$ of the map $\mF^\mathrm{aff}_G$ are linear spaces, they are irreducible. 
So $\mF_G^\mathrm{aff}$ is an algebraic map between irreducible varieties and Corollary \ref{lemma:jacobian_fiber} implies that:
$$ \dim {\mF_G^\mathrm{aff}}^{-1}(\mF_G^\mathrm{aff}(\mP_0)) = 11|V|
- \mathrm{rank} \jac{\mF_G^\mathrm{aff}} (\mP_0). $$
Now, finite solvability means that the generic non-empty fiber is finite modulo $\mathrm{PGL}_4$, i.e., the fiber is a union of finitely many copies of $\mathrm{PGL}_4$.  Since the latter group has dimension $15$, we obtain:
$$ \text{finite solvable }
\iff \dim {\mF_G^\mathrm{aff}}^{-1}(\mF_G^\mathrm{aff}(\mP_0)) =15,
$$
hence we get the thesis.
\end{proof}

\subsection{Our Formulation}

Our formulation employs a polynomial system where the only unknowns are the camera matrices, by using an implicit homogeneous constraint that links fundamental matrices to cameras (from \cite[Chap.\ 9]{HartleyZisserman04}). 
%
With respect to using the explicit map $\mF$ as defined in \eqref{eq:Fuv_det}, which is indeed theoretically feasible, this approach yields lower-degree polynomials and eliminates the need to account for the projective scales.

\begin{lemma}[Result 9.12 in \cite{HartleyZisserman04}]
A non-zero matrix $F_{ij}$ is the fundamental matrix corresponding to a pair of cameras $P_i$ and $P_j$ 
if and only if the matrix $  S \coloneqq  P_j^\tr   F_{ij} P_i  $ is skew-symmetric. 
\label{lemma_skew}
\end{lemma}

Note that any scaling of each of the three terms of the product would clearly leave the result skew-symmetric. The above condition can be rewritten as:
\begin{equation}
   S+S^\tr = 0
  \quad \Longleftrightarrow \quad 
   P_j^\tr   F_{ij} P_i  + P_i^\tr   F_{ij}^\tr  P_j = 0.
   \label{eq:skewform}
\end{equation}
Since \eqref{eq:skewform} is symmetric, it 
translates into 10 quadratic equations when considered entry-wise. 
Observe that these equations have not been used in previous works on solvability \cite{TragerOssermanAl18,ArrigoniFusielloAl21,ArrigoniPajdlaAl23}.

We write $\mathrm{Sym}_4$ for the vector space of real symmetric $4 \times 4$ matrices, and $\P\, \mathrm{Sym}_4$ for its projectivization.
Note that the latter is isomorphic to $\P^9$ since $\dim \mathrm{Sym}_4 = 10$.
Let us define:
\begin{align}
    \Phi: \P^{11} \times \P^{11} \times \P^{8} &\to \P\, \mathrm{Sym}_4 \cong \mathbb{P}^{9}, \label{eq:skewform_J}
    \\
    (P_i,P_j,F) &\mapsto P_j^\tr F P_i + P_i^\tr F^\tr P_j. \nonumber
\end{align}
Note that $\Phi$ is homogeneous in each of its inputs.
 Lemma~\ref{lemma_skew} states that there is a unique $F$ (in the projective space) such that $\Phi(P_i,P_j,F)=0$, and this is the fundamental matrix corresponding to the camera  pair ($P_i,P_j$). In formulae: 
\begin{equation}
    \Phi(P_i, P_j, F) = 0 \; \iff   F  = \mF(P_i, P_j) = F_{ij} .
    \label{eq:lemma4}
\end{equation} 

Note that Eq.~\eqref{eq:skewform} holds for a single edge $(i,j) \in {E}$. By collecting equations coming from all the edges in the graph $G$, it results in a \emph{polynomial system} 
\begin{equation}
\Phi_G((P_i)_{i \in V}, (F_e)_{e \in E}) =0.
\label{eq:polysystem}
\end{equation}
with 
$\Phi_G: (\P^{11})^{|V|} \times (\P^{8})^{|E|} \to (\P^{9})^{|E|}$.
%
Specifically, since we start from fundamental matrices given by a generic configuration $\mP_0$,  our polynomial system is 
\begin{equation} 
\Phi_G(\mP, \mF_G(\mP_0)) =0
\end{equation} 
with unknowns $\mP$.
It is clear from the definitions (and Lemma~\ref{lemma_skew}) that 
this system has a unique solution (equal to $ \mP_0$) if and only if $\mF_G^{-1}(\mF_G(\mP_0)) = \{\mP_0\}$ (modulo $\mathrm{PGL}_4$), which is tantamount to saying that  $G$ is solvable.

{Similarly to before, we restrict the maps $\Phi$ and  $\Phi_G$ to affine charts.
But since $\Phi$ vanishes on corresponding camera pairs and their fundamental matrices, this time we do not restrict the codomain to an affine chart (otherwise, we would work with fractions with vanishing denominator; cf. footnote${}^{\ref{footnoteAffineChart}}$). 
We denote the affine versions of our maps by  $\Phi^\mathrm{aff}: (\RR^{11})^2 \times \RR^8 \to \RR^{10} $ and 
$\Phi^\mathrm{aff}_G: (\RR^{11})^{|V|} \times (\RR^8)^{|E|} \to (\RR^{10})^{|E|} $.}
The following proposition links the rank of $\jac{\mF_G^\mathrm{aff}}$
to that of  the Jacobian 
of $\Phi_G^\mathrm{aff}$ with respect to cameras, thereby establishing an alternative characterization of finite solvability, which we formalize later in Theorem \ref{teo:theOne}.

\begin{prop}
\label{prop:rankequality}
The map $\mF_G^\mathrm{aff}$ is implicitly defined by $\Phi_G^\mathrm{aff}$ (in a neighborhood of a solution).
Moreover, 
for a generic configuration $\mP_0$ with fundamental matrices $\mF_0 := \mF_G^\mathrm{aff} (\mP_0)$, we have: 
\begin{equation} \label{eq:rankEquality}
\rank\left(\dfrac{\partial \Phi_G^\mathrm{aff}}{\partial (P_i)_{i \in V}} (\mP_0, \mF_0) \right) = \rank(\jac{\mF_G^\mathrm{aff}}(\mP_0)) .
\end{equation}
\end{prop}

\begin{proof}
Consider the function $\Phi$ defined in \eqref{eq:skewform_J}.
The Jacobian\footnote{{In fact, this is the Jacobian of 
$\tilde{\Phi}: (\mathbb{R}^{12})^{2} \times \mathbb{R}^{9} \to \mathbb{R}^{10}$,
 the homogeneous map that induces ${\Phi}$ by identifying collinear points within their projective equivalence classes. However, we omit this distinction to maintain notational simplicity.}} 
  of  $\Phi$, denoted by $\jac \Phi $, can be reorganized as:
\begin{center}
\begin{tikzpicture}
    \draw (0, 0) rectangle (3, 2); 
    \node at (1.5, 1) {$\dfrac{\partial \Phi}{\partial P_i, P_j}$};   
    \node[above] at (1.5, 2) {\footnotesize{24}};   
    \node[left] at (0, 1) {$\jac \Phi = $  };   
    \draw (3, 0) rectangle (5, 2); 
    \node at (4, 1) {$\dfrac{\partial \Phi}{\partial F}$};   
    \node[above] at (4, 2) {\footnotesize{9}};   
    \node[right] at (5, 1) {\footnotesize{10}};   
\end{tikzpicture}    
\end{center}
Since $\Phi$ is linear in $F$, this means that, for fixed cameras $P_i$ and $P_j$ with distinct centers, the $10 \times 9$ matrix $\frac{\partial \Phi}{\partial F}$ has rank 8 with kernel given by $\mathrm{span} \{ F \}$.

Since  $\Phi$ is homogeneous in each of its inputs, we can think of the matrices $P_i,P_j,F$ and $\Phi(P_i,P_j,F)$ in their respective projective spaces instead.
Recall that the tangent space of $\mathbb{P}^8$ at $F$ is the quotient vector space $\mathbb{R}^{3 \times 3}/\mathrm{span}\{ F\}$ \cite[Chap.\ 2]{Shafarevich13}. 
That means, when restricting the domain of $\Phi$ to affine charts, which we denote by $\Phi^\mathrm{aff}: \mathbb{R}^{11} \times \mathbb{R}^{11} \times \mathbb{R}^{8} \to \mathbb{R}^{10}$,
then the $10 \times 8$ Jacobian matrix $\frac{\partial \Phi^\mathrm{aff}}{\partial F}$ is of full rank $8$.

To turn this into an invertible matrix, we fix a generic $8 \times 10$ matrix $A$ and consider the composition 
$A {\circ} \Phi^\mathrm{aff}: \mathbb{R}^{11} \times \mathbb{R}^{11} \times \mathbb{R}^{8} \to \mathbb{R}^{8}$.
Since the Jacobian of $\Phi^\mathrm{aff}$ is divided into two blocks of size $10 \times 22$ and $10 \times 8$, then the Jacobian of $A {\circ }\Phi^\mathrm{aff}$ has the following structure:
\begin{center}
    \begin{tikzpicture}
    \draw (0, 0) rectangle (3, 2); 
    \node at (1.5, 1) {$  \dfrac{\partial A {\circ} \Phi^\mathrm{aff}}{\partial P_i, P_j}$};   
    \node[above] at (1.5, 2) {\footnotesize{22}};   
    \node[left] at (0, 1) {$\jac A {\circ } \Phi^\mathrm{aff} = $ };   
    \draw (3, 0) rectangle (5, 2); 
    \node at (4, 1) {$ \dfrac{\partial A {\circ} \Phi^\mathrm{aff}}{\partial F}$};   
    \node[above] at (4, 2) {\footnotesize{8}};   
    \node[right] at (5, 1) {\footnotesize{8}};   
\end{tikzpicture}
\end{center}


Since now $ \frac{\partial A {\circ} \Phi^\mathrm{aff}}{\partial F}$ is invertible, we can apply the Implicit Function Theorem: there is a function $f$ defined and differentiable  in some neighborhood of a  solution, such that
$ A \circ \Phi^\mathrm{aff}(P_i, P_j, f(P_i, P_j))= 0 $. {This is not a surprise as we already know that $\mF^\mathrm{aff}$ is this function $f$.
However, the Implicit Function Theorem also tells us that the Jacobian of the function $f= \mF^\mathrm{aff} $} is given by
\begin{equation}
\frac{\partial \mF^\mathrm{aff}}{\partial P_i, P_j}  = -\left(\frac{\partial A {\circ} \Phi^\mathrm{aff}}{\partial F}\right)^{-1}  \frac{\partial A {\circ} \Phi^\mathrm{aff}}{\partial P_i, P_j}.
\label{eq:implicitFunctionTheoremF}
\end{equation}
{(Note that the Jacobian matrices in this equality should be evaluated at a generic camera pair and their corresponding fundamental matrices, just as in \eqref{eq:rankEquality}, but we skip this here for simpler notation.)}
Since $\frac{\partial A {\circ} \Phi^\mathrm{aff}}{\partial F}$ is invertible, 
for a generic camera pair, the ranks of $\frac{\partial \, A {\circ} \Phi^\mathrm{aff}}{\partial P_i, P_j}$ and $\frac{\partial \mF^\mathrm{aff}}{\partial P_i, P_j}$ are the same. 
{Due to the genericity of the matrix $A$, this rank is the same as} 
 $ \mathrm{rank} \frac{\partial \Phi^\mathrm{aff}}{\partial P_i, P_j}$.

Observe that, by \eqref{eq:implicitFunctionTheoremF}, $\frac{\partial \, A {\circ} \Phi^\mathrm{aff}}{\partial F}$ serves as a local coordinate change between the explicit coordinates of each fundamental matrix and its implicit coordinates in terms of the skew-symmetric matrix condition, as in the following diagram:
\begin{equation}\begin{tikzcd}
	\RR^{11} \times \RR^{11} && \RR^{8} \\
	& \RR^{8}
	\arrow["\frac{\partial \mF^\mathrm{aff}}{\partial P_i, P_j} ", from=1-1, to=1-3]
	\arrow["\frac{\partial A {\circ} \Phi^\mathrm{aff}}{\partial P_i, P_j}" ', from=1-1, to=2-2]
	\arrow["\frac{\partial \, A {\circ} \Phi^\mathrm{aff}}{\partial F}", from=1-3, to=2-2]
    \arrow["", from=2-2, to=1-3, shift left]
\end{tikzcd}\end{equation}

\bigskip

Finally, we consider the map $\mF^\mathrm{aff}_G: (\mathbb{R}^{11})^{|V|} \to (\mathbb{R}^{8})^{|E|}$ that computes the fundamental matrices along a viewing graph $G$ with nodes $V$ and edges $E$.
Similarly, we extend the function $A {\circ} \Phi^\mathrm{aff}$ to 
$(A \times \ldots \times  A) {\circ} \Phi^\mathrm{aff}_G: (\mathbb{R}^{11})^{|V|} \times (\mathbb{R}^{8})^{|E|} \to (\mathbb{R}^{8})^{|E|}$.
{This gives local coordinate changes between the explicit and implicit coordinates of the fundamental matrices along each of the edges, and so as above we obtain:}
\begin{align*}
     \frac{\partial \mF^\mathrm{aff}_G}{\partial (P_i)_{i \in V}} = - \left(\frac{\partial \, A^{(|E|)} {\circ} \Phi^\mathrm{aff}_G}{\partial (F_e)_{e \in E}}\right)^{-1} \frac{\partial \, A^{(|E|)} {\circ} \Phi^\mathrm{aff}_G}{\partial (P_i)_{i \in V}}.
\end{align*}
Hence,
$\mathrm{rank} \dfrac{\partial \mF^\mathrm{aff}_G}{\partial (P_i)_{i \in V}} = \mathrm{rank} \dfrac{\partial \Phi^\mathrm{aff}_G}{\partial (P_i)_{i \in V}}$ for generic  $P_i$.
\end{proof}

Finally, we establish our \textbf{main result}, which was demonstrated in only one direction in our conference paper \cite{ArrigoniFusielloAl24}.
\begin{theorem} A graph $G$ is finite solvable if and only if 
$ \rank\left(\dfrac{\partial \Phi_G^\mathrm{aff}}{\partial (P_i)_{i \in V}} (\mP_0, \mF_0) \right)  = 11|V| - 15 $ for a generic configuration $\mP_0$ with fundamental matrices $\mF_0 := \mF_G^\mathrm{aff} (\mP_0)$.
\label{teo:theOne}
\end{theorem}
\begin{proof}
It follows from Propositions \ref{lemma:Jcheck->solvable} and  \ref{prop:rankequality}.
\end{proof}

The condition on the rank of the Jacobian, often referred to as the ``Jacobian check'' in Algebraic Vision \cite{DuffKohnAl24}, ensures that solutions are finite in the neighborhood of isolated solutions but does not provide guarantees for other fibers. 
The question of whether a statement can be made about almost all fibers was left open in \cite{ArrigoniFusielloAl24}. Theorem \ref{teo:theOne} provides the answer, which ultimately relies on the Fiber Dimension Theorem. Figure \ref{fig:big_picture} provides a summary of our results as well as connections between various solvability notions.

\begin{figure}[htbp]
    \centering
    \includegraphics[scale=.2 ]{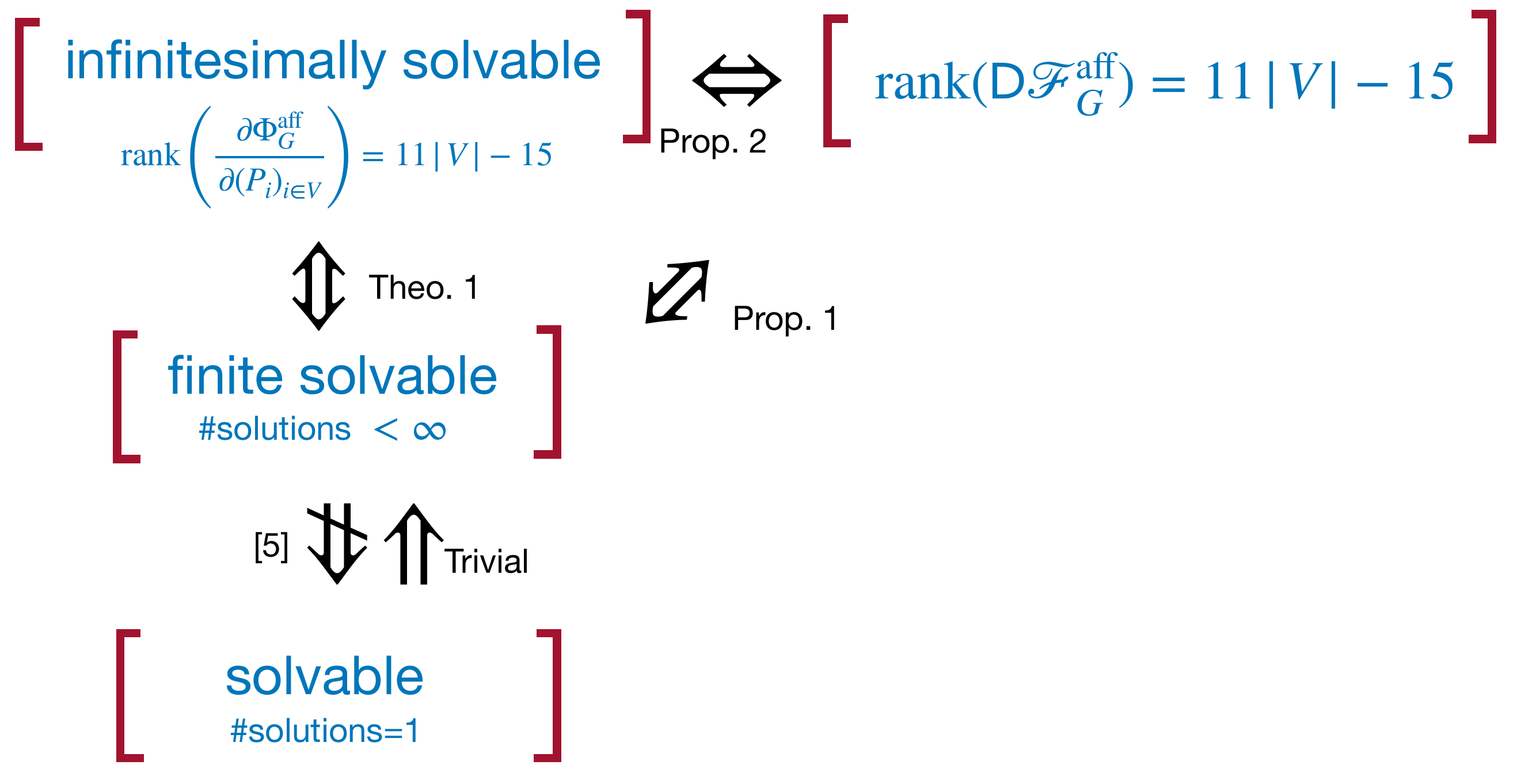}
    \caption{The connections among different concepts of solvability. The rank condition in Theorem \ref{teo:theOne} was called ``infinitesimally solvable'' in \cite{ArrigoniFusielloAl24}. }
    \label{fig:big_picture}
\end{figure}

\section{Proposed Method}
\label{sec_method}

In this section we show how Theorem \ref{teo:theOne} can be used in practice to test finite solvability. 
We also show how to partition an unsolvable graph into maximal subgraphs (called components) that are finite solvable. 

\subsection{Testing Finite Solvability}
\label{sec:testing}

The conclusion of the previous section is that, in order to establish  finite solvability of a viewing graph $G=(V,E)$, one can test if 
$r \coloneqq \rank(\frac{\partial \Phi_G^\mathrm{aff}}{ \partial (P_i)_{i \in V}})$ is equal to $ 11|V| - 15$, 
{where this $10|E| \times 11|V|$-Jacobian is evaluated at a configuration $\mP_0$ and fundamental matrices  
$\mF_G^\mathrm{aff}(\mP_0)$  (please note that these fundamental matrices are compatible by construction).}
For computational reasons (that will be clarified in the end of this section), it is preferable to test whether a matrix is full rank rather than determining its exact rank. Therefore, we include 15 additional independent equations in order to fix a basis for $\mathrm{PGL}_4$, which raises the rank by 15 (making it full rank if and only if $G$ is finite solvable).

In practice, we 
consider the 
 Jacobian  $J_P \coloneqq  \frac{\partial \Phi_G}{\partial (P_i)_{i \in V}}$, 
which has dimension  $10|E| \times 12|V|$  and rank $r$.
{The rank is the same as above because it is the codimension of the tangent space of the variety defined by $\Phi_G(\mP, \mF_G(\mP))=0$ at $\mP_0$, and that codimension is the same no matter whether one looks at an affine chart or the affine cone over the projective variety.}
The affine chart is fixed by introducing one additional equation per camera, which raises the rank by $|V|$. Overall, the Jacobian $J$  of the augmented polynomial system has  rank 
$r + 15 + |V|$ and 
 it achieves full-rank $12|V|$ if and only if $r  = 11|V| - 15$.

Specifically, the global \emph{projective ambiguity} is fixed, without loss of generality, by arbitrarily choosing the first camera and the first row in the second camera:
\begin{equation}
P_1=\begin{bmatrix}
I_{3\times3} & 0_{3\times1} 
\end{bmatrix}
\ \  \text{and} \ \ 
\left[
\begin{array}{@{\,}c@{\,}c@{\,}c@{\,}}
1 & 0 & 0
\end{array}\right]
P_2 = 
\left[
\begin{array}{@{\,}c@{\,}c@{\,}c@{\,}c@{\,}}
0 & 0 & 0 & 1
\end{array}\right],
\label{eq_ambiguity}
\end{equation}
resulting in {16 additional equations}.
Note that $[1 \ 0 \ 0] P_2$ is equivalent to selecting the first row in $P_2$. In fact, any pair of cameras can be chosen to fix the projective ambiguity. In practice, we will use two nodes that are endpoint of an edge in the graph (see Section \ref{sec_components}).

Concerning the selection of the \emph{affine chart}, the scale of each camera can be arbitrarily set, e.g., by fixing the sum of its entries to 1: 
    \begin{equation}
       \mathbf{1}_{12}^{\mathsf{T}} \text{vec}(P_{i}) = 1,
       \label{eq_scale}
    \end{equation}
where $\mathbf{1}_{12}$ denotes a vector of ones of length 12. This results in a linear equation for each node, except the first camera used to fix the global ambiguity. In total we add $16 + |V| -1 =  15 + |V|$ equations.
In summary, equations of the form \eqref{eq:polysystem} (i.e., $\Phi_G =0$), \eqref{eq_ambiguity} and \eqref{eq_scale} are all collected in a polynomial system, for a total of
$10|E| + (|V|-1) + 16 = 10|E| + |V| + 15$ equations.
The unknowns of our polynomial system are the camera matrices, for a total of $12|V|$ unknowns.

The Jacobian matrix of our polynomial system -- denoted by $J$ -- 
contains $J_P$ and the derivatives of \eqref{eq_ambiguity} and \eqref{eq_scale}.
$J_P$ is constructed by $10 \times 12$ blocks, whose formulas are given in Section \ref{sec_implementation} -- see Equations \eqref{eq:blockA} and \eqref{eq:blockB}. 
The block structure follows the incidence matrix $B$ of the viewing graph, which has one row for every edge and one column for every node.  
In the row of $B$ that represents the edge $(i,j)$, there is a $-1$ in column $i$ and a $+1$ in column $j$ and other entries are zero.  In $J_P$, the $+1$ is replaced by  the $10 \times 12$ block 
\eqref{eq:blockA} and  the $-1$ is replaced by the  $10 \times 12$ block \eqref{eq:blockB}: 
\begin{equation}
   \begin{bmatrix} 0 \cdots 0, 
  \dfrac{\partial  \Phi }{\partial ( \veco P_i  )^\tr }
  \,  , 0 \cdots 0,  \,
   \dfrac{\partial \Phi }{\partial ( \veco P_j  )^\tr } ,  0 \cdots 0
 \end{bmatrix}   .
 \label{eq_jacobian_block}
\end{equation}
Matrices of the form \eqref{eq_jacobian_block} are then  stacked for all the edges in the graph 
to make $J_P$. Note that this matrix is sparse  because B is sparse.
As for the derivatives of the  additional equations that fix scales and projective ambiguity, they are constant matrices of zero and ones.

To summarize, in our implementation:
\begin{enumerate}
    \item we assign random cameras $\mP_0$ to nodes of $G$ and compute the fundamental matrices using $\mF_G(\mP_0)$;
    \item we then  build the Jacobian as just explained;
    \item we test finite solvability by checking whether $J$ is of full rank.
\end{enumerate}
The last point is accomplished by computing the smallest singular value of $J$, which in turn is equivalent to computing the smallest eigenvalue of $J^{\mathsf{T}}J$. Checking a given rank, instead, would entail computing more eigenvalues: a number equal to the kernel dimension.

\begin{remark} 
{Theorem \ref{teo:theOne} applies to generic camera configurations, meaning that we establish if the number of generic solutions of  $\Phi_G$ is finite or not. However, additional non-generic solutions may exist -- for example, those corresponding to rank-deficient cameras, that symbolic solvers  will find. If one is interested in counting all generic solutions then one should incorporate extra equations into the polynomial system to enforce full-rank camera conditions, as was done in \cite{ArrigoniFusielloAl24}.}   
\end{remark}

\begin{figure}[t]
\centerline{
\subfloat[Three components]{
\includegraphics[width=0.49\linewidth]{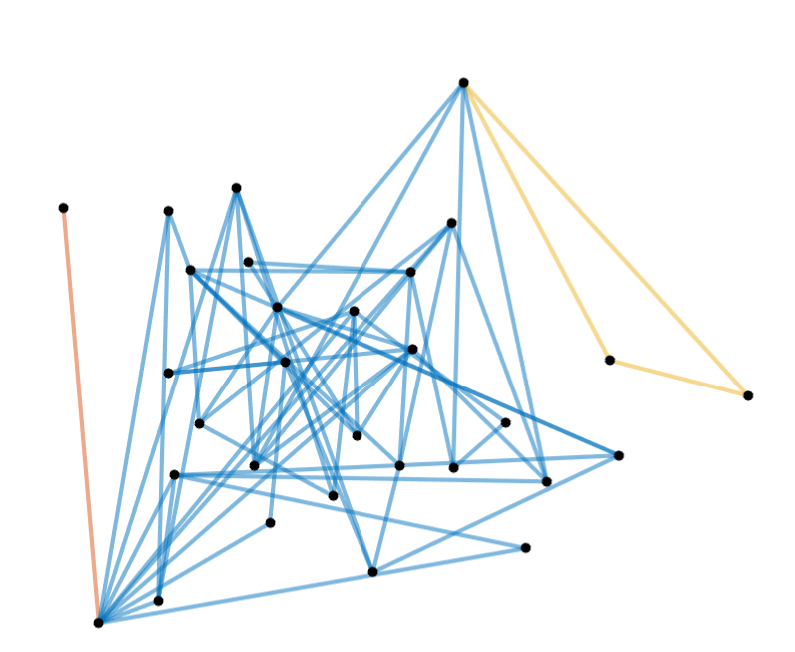}
\label{fig:comp3}
}
\hfil
\subfloat[Four components]{
\includegraphics[width=0.49\linewidth]{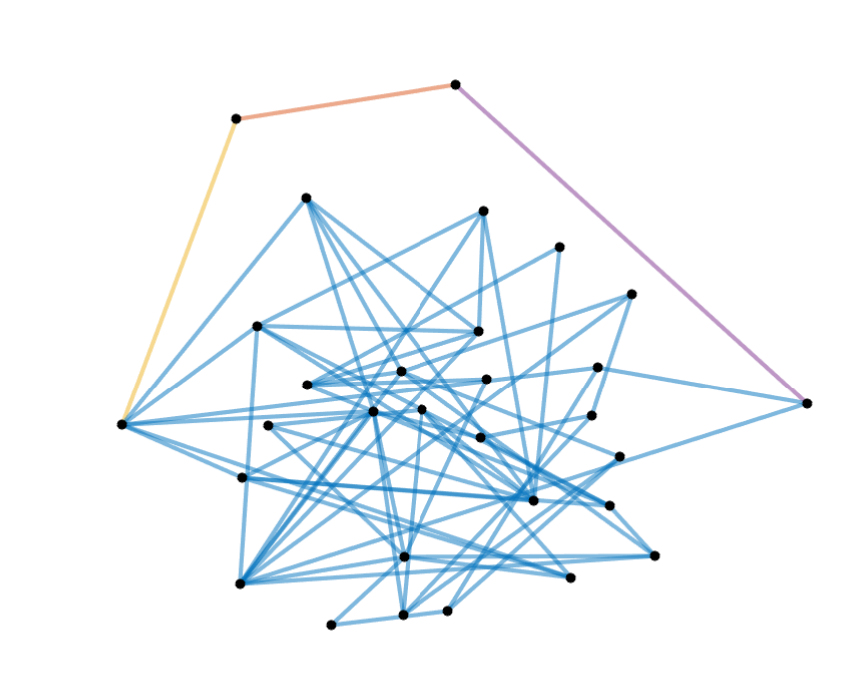}
\label{fig:comp4}
}}
\centerline{
\subfloat[Six components]{
\includegraphics[width=0.49\linewidth]{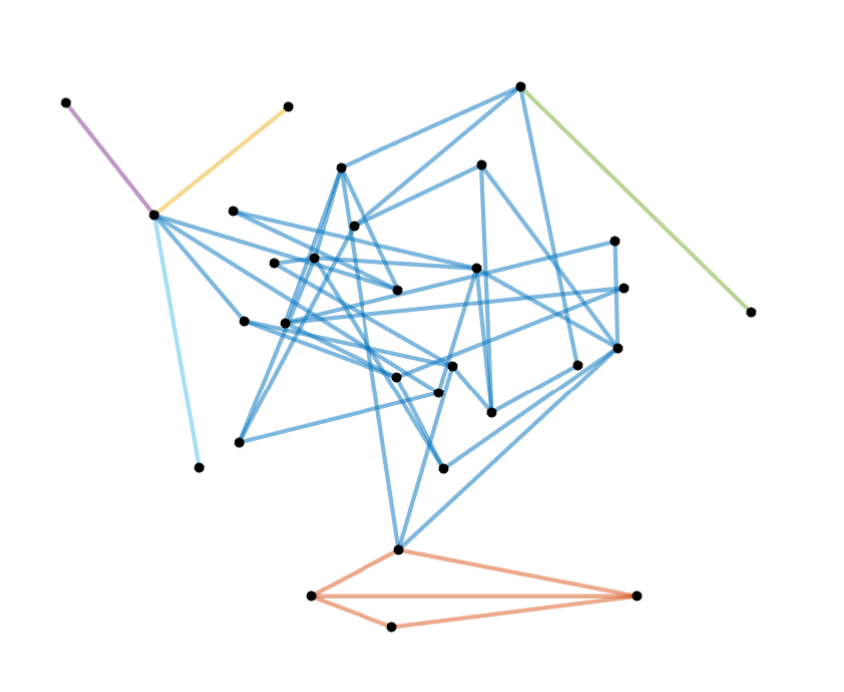}
\label{fig:comp6}
}
\hfil
\subfloat[Seven components]{
\includegraphics[width=0.49\linewidth]{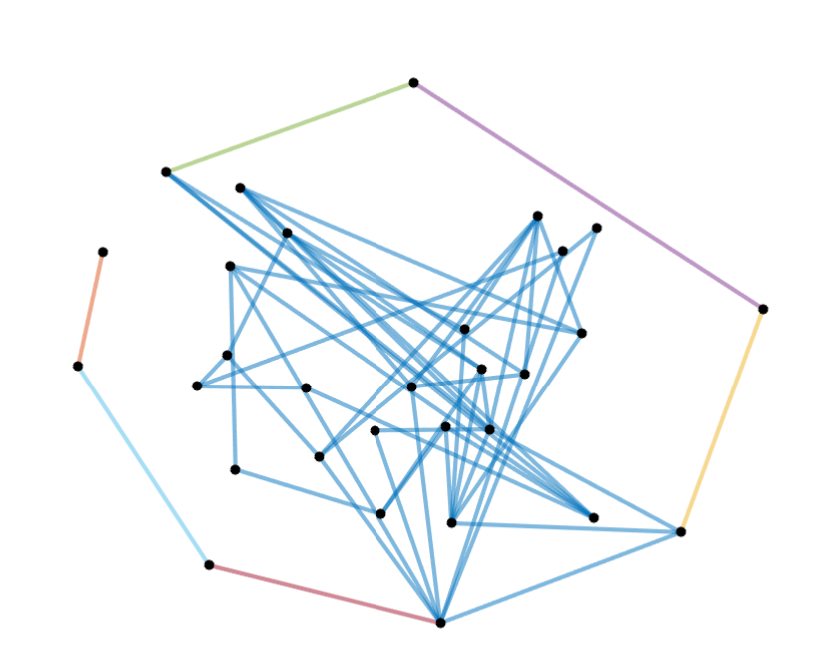}
\label{fig:comp7}
}
}
\caption{Examples of maximal components on synthetic viewing graphs, where each component is represented with a different color. }
    \label{fig:inclusion_components}
\end{figure}

\subsection{Finding Maximal Components}
\label{sec_components}

We now show how to extract maximal finite-solvable components in the case where a viewing graph is established to be unsolvable. 
Proposition 3 from \cite{ArrigoniPajdlaAl23} states that such components form a \emph{partition} of the edges. In other terms, each edge belongs to exactly one component. However, it is important to remark that a node can belong to more components. 
For example, an articulation point (or cut vertex) always belongs to two different components (see e.g. the blue and red components of Figure~\ref{fig:comp3} which share a cut vertex).

For this reason, we can not trivially use the \emph{edge-based} methodology from \cite{ArrigoniPajdlaAl23}, for our formulation  is \emph{node-based} (i.e., our unknowns are associated with the nodes in the graph). Indeed, the null space of the Jacobian matrix $J$, which is nontrivial for an unsolvable graph, is such that any block of 12 rows corresponds to a camera/node in the viewing graph (whereas in \cite{ArrigoniPajdlaAl23} there was a correspondence between rows in the null space and edges in the graph). Luckily, it is still possible to identify components from the null space of $J$ with proper modifications. In this context, an important observation is that we need two nodes to fix the global projective ambiguity: in particular, for the purpose of identifying components, it is useful to select two adjacent nodes (i.e., an edge).

\begin{lemma}
Let $J$ be the Jacobian matrix constructed as explained above, and let $N$ be the null space of $J$. A node is in the same component as the edge used to fix the global projective ambiguity $\Longleftrightarrow$ the associated rows in $N$ are zero. 
\end{lemma}
\begin{proof}
Following the reasoning from \cite{ArrigoniPajdlaAl23}, we can prove the thesis based on this observation:  if we focus on the component containing the edge used to fix the global projective ambiguity, then the fact that all ambiguities have been fixed in that component, it is equivalent to say that there are no degrees of freedom, i.e., the null space is trivial on that component. 
\end{proof}

According to the above result, we can define an iterative approach to identify components: 
\begin{itemize}
    \item first, we use two adjacent nodes to fix the global projective ambiguity and identify all nodes within such component by selecting the ones corresponding to the zero rows in $N$; 
    \item then, we repeatedly apply the same procedure to the remaining part of the graph until there are no more edges to be assigned.
\end{itemize}
Observe that, although the null space is computed several times (equal to the number of components), the Jacobian matrix has a size that gets smaller and smaller, for the procedure is not re-applied to the whole graph but only to the subgraph containing the remaining edges.

Some visualizations of the components are given in Figure~\ref{fig:inclusion_components}, showing established cases of unsolvable graphs, like the ``square'' topology (\ref{fig:comp4},\ref{fig:comp7}) or the presence of pendant edges and articulation points (\ref{fig:comp3},\ref{fig:comp6},\ref{fig:comp7}).

\section{Formulas for Derivatives}
\label{sec_implementation}

In this section we report explicit formulas for the derivatives of our polynomial equations with respect to their unknowns, that are the basis of the Jacobian check implemented within our approach.

\subsection{General Formulas}

The derivatives of functions involving vectors and matrices ultimately
lead back to the partial derivatives of the individual components, and
it is all about how to arrange these partial derivatives. There are several conventions, here we follow \cite{MagnusNeudecker99}, which allows to apply the chain rule.

\begin{definition}[\cite{MagnusNeudecker99}]
	Let $f: \RR^{s\times t} \to \RR^{r \times q}$  be a differentiable function. The derivative of $f$ in $X$ is the matrix $rq \times st$:
	\begin{equation}
		\label{eq:jac}
		\jac f (X)= \frac{\partial \veco f(X)}{\partial(\veco X)^\tr}.
	\end{equation}
where $\veco$ denotes the vectorization of the matrix by stacking its columns. In particular, if $f: \RR^s \to \RR^q$  then  $ \jac f(\vx) $ coincides with the usual \emph{Jacobian matrix} of $f$. 
\end{definition}  

We will use the following lemma \cite{MagnusNeudecker99}.

\begin{lemma}
Assuming $A, X$ and $B$ are matrices of sizes $r \times s $, $ s \times t$ and $ t \times q$, respectively, then the derivatives of the following matrix functions in $X$ are:
\begin{equation}
\begin{gathered}
     \jac(AX) = (I_t \kron A) \\    
      \jac(XB) = (B^\tr \kron I_s) \\
   \jac(AXB) = (B^\tr \kron A) 
\end{gathered}
\label{eq:jacAXB} 
\end{equation}
where $\otimes$ denoted the Kronecker product.
\label{lemma_derivatives}
\end{lemma}

The above result exploits the  ``vectorization trick'' \cite{HendersonSearle81}, stating that we can write  $ \veco(AXB) = (B^\tr \kron A) \veco X$. 
Furthermore, it can be also shown that \cite{MagnusNeudecker99}:
\begin{equation}
\jac(X^\tr) = K_{s,t}    \label{eq:jactranspose}
\end{equation}
where $K_{s,t}$ is the  \emph{commutation matrix}, namely the $st \times st $  matrix  such that $ \veco (C) = K_{s,t} \veco (C^\tr)$  for any $C$ of size $ {s \times t}$. 
Moreover, for $A$ and $B$ of sizes $ {r \times s}$ and ${t \times q} $ respectively, we have \cite{MagnusNeudecker99}:
\begin{equation}
    B \kron A = K_{r,t} (A \kron B) K_{q,s} .
\end{equation}
The commutation matrix is a permutation, hence it is orthogonal: $  K_{s,t}  K_{s,t}^\tr = I_{st}  $. 
Note also that $K_{s,t} =  K_{t,s}^\tr$.

When vectorizing symmetric matrices, the $\vech$ operator is used to extract only the lower triangular part of the matrix. 
There exist unique matrices transforming the half-vectorization of a matrix to its vectorization and vice versa called, respectively, the \emph{duplication matrix} and the \emph{elimination matrix} \cite{MagnusNeudecker99}. In particular for the latter we have: $ \vech(X)  =  L_q  \veco(X) \quad  \forall X \in \mathrm{Sym}_q(\mathbb{R}).$

The reader is referred to  \cite{HendersonSearle81} for a review of results involving the Kronecker product and \cite{MagnusNeudecker99}
 for derivatives of matrix functions.

\subsection{Derivatives of $\Phi$ with Respect of Cameras}

Consider the map $\Phi$  defined in Eq.~\eqref{eq:skewform_J}. Since the codomain is given by symmetric matrices, we consider -- in practice --  only the lower triangular part:
\begin{equation}
   \Phi(P_i, P_j, F) \coloneqq  \vech \left( P_j^\tr   F  P_i + P_i^\tr   F_{ij}^\tr  P_j\right) 
\end{equation}
where the $\vech$ operator vectorizes while extracting the non-duplicated entries.
Let us define $ S \coloneqq  P_j^\tr F_{ij}  P_i $, then the map $\Phi$ rewrites:
\begin{equation}
\Phi(P_i, P_j, F) =
   \vech(S+S^\tr) .
   \label{eq:skewform_2}
\end{equation}
From the formulae \eqref{eq:jacAXB} and \eqref{eq:jactranspose}  we get:
\begin{align}
 \frac{\partial \veco S }{\partial ( \veco P_i  )^\tr }
 &= I_4 \kron ( P_j^\tr F_{ij} ) \\ 
 \frac{\partial \veco S^\tr }{\partial ( \veco P_i )^\tr }
 &= K_{4,4} (I_4 \kron ( P_j^\tr F_{ij} ) )\\
\frac{\partial \veco S }{\partial ( \veco P_j  )^\tr }
 &= 
 K_{4,4} (I_4 \kron ( F_{ij} P_i)^\tr )\\
\frac{\partial \veco S^\tr }{\partial ( \veco P_j  )^\tr }
 &= 
 I_4 \kron  ( F_{ij} P_i)^\tr
 .
\end{align}
Hence:
\begin{align}
 \frac{\partial \Phi }{\partial ( \veco P_j  )^\tr } & = L_4
(K_{4,4} + I_{16}) \left( I_4 \kron (  F_{ij} P_i)\!^\tr \right) \label{eq:blockA}\\
 \frac{\partial \Phi }{\partial ( \veco P_i )^\tr } & = L_4
 (K_{4,4} + I_{16}) \left( I_4 \kron ( P_j^\tr F_{ij} ) \right) \label{eq:blockB}
\end{align}
where $L_4$ is the elimination  matrix.
Hence the Jacobian of $\Phi$ with respect to vectorized cameras $\veco(P_j)$ and $\veco(P_i)$ is the following 10 $\times$ 24 matrix:
\begin{equation}
\begin{split}
& \frac{\partial \Phi}{\partial [(\veco P_i)^\tr| (\veco P_j )^\tr]} = \\[5pt] = & L_4(K_{4,4} {+} I_{16}) \begin{bmatrix}
   (I_4 {\kron} (  F_{ij} P_i)^\tr)
 \, |  \,
  (I_4 {\kron} ( P_j^\tr F_{ij} ) )
\end{bmatrix}
\label{eq:jacobian}.
\end{split}
\end{equation}
In the rest of the manuscript, the same Jacobian matrix has been referred to, with a slight abuse of notation, as 
$ \frac{\partial \Phi}{\partial P_i, P_j}.$

\subsection{Derivatives  of $\Phi$ with Respect to Fundamental Matrices}

By the definition of elimination matrix and commutation matrix one can rewrite 
\eqref{eq:skewform_2}  as
\begin{equation}
  \Phi(P_i, P_j, F) =
L_4 ( I_{16}  + K_{44}) \veco(S) .   
\end{equation}
Since
\begin{equation}
\veco(S) = \veco(P_j^\tr   F_{ij} P_i) =  
(P_i^\tr \kron P_j^\tr ) \veco(F)    
\end{equation}
we get
\begin{equation}
\Phi(P_i, P_j, F) =\underbrace{L_4 ( I_{16}  + K_{44})}_{10 \times 16}(P_i^\tr \kron P_j^\tr ) \veco(F) .    
\end{equation}
Hence, the the Jacobian of $\Phi$ with respect to vectorized fundamental matrices can be obtained as the following 10 $\times 9$ matrix:
\begin{equation}
\frac{\partial \Phi }{\partial ( \veco F  )^\tr }  =   L_4(K_{4,4} + I_{16})(P_i^\tr \kron P_j^\tr ) .
\end{equation}
%
In the rest of the manuscript, the same Jacobian matrix has been referred to, with a slight abuse of notation, as $  \frac{\partial \Phi}{\partial F}.$

\section{Experiments}
\label{sec_experiments}

In this section we report results on synthetic and real data. 
We implemented our method in MATLAB R2023b -- the code is publicly available\footnote{\url{https://github.com/federica-arrigoni/finite-solvability}} -- and used a MacMini M1 (2020) with 16Gb RAM for our experiments. 
We compared our approach to the one by Arrigoni et al.~\cite{ArrigoniPajdlaAl23}, that addresses finite solvability as well. 
We also discuss the efficiency of the analyzed approaches.
We did not consider the method by Trager et al.~\cite{TragerOssermanAl18} in our comparisons since it was subsumed by \cite{ArrigoniPajdlaAl23}.
We refer the reader to \cite{ArrigoniPajdlaAl23} and Section \ref{sec_size_matrices} for additional insights on the performance of \cite{TragerOssermanAl18}.

\subsection{Mining minimally solvable graphs}

A graph is called \emph{minimally solvable} if the removal of any edge results in a non-solvable graph. Recall that a necessary condition is that the graph has at least $(11n-15)/7$ edges, with $n$ being the number of nodes \cite{TragerOssermanAl18}. So, any solvable graph with this number of edges is minimally solvable.
In the need of a combinatorial characterization of minimally solvable graphs, we can build a catalog by ``mining'' them. 
In this respect, we exhaustively generated all the biconnected graphs (a necessary condition for solvability) with a given number of nodes (up to ten) having $\lceil (11 n -15) / 7 \rceil $ edges. Among these candidates, we tested for finite solvability using our method and \cite{ArrigoniPajdlaAl23}, which always returned the same result, as expected. The results are reported in Table \ref{tab:minimal_graphs}. 
Please note that among the 27 minimal graphs with 9 nodes that passed the test, there are the 10 counterexamples found by \cite{ArrigoniFusielloAl21} that are finite solvable and admit two realizations in $\RR$, showing that finite solvability $\nRightarrow$ solvability.

\begin{table}[ht]
        \centering
\caption{Minimally solvable graphs. \\ Column
``\#candidates'' reports the number of  biconnected graphs with up to 10 nodes  and  $\lceil  (11 n -15) / 7  \rceil $ edges; out of these graphs, ``\#fin\_solv'' have been found  finite solvable. }
         \begin{tabular}{@{}r r r @{}}
    \toprule
         $\#$nodes & \#candidates& \#fin\_solv\\ \midrule
         3&  1& 1\\
         4&  1& 1\\
         5&  2& 1\\
         6&  9& 4\\
         7&  20& 3\\
         8&  161& 36\\
         9&  433& 27\\
         10&  5898& 756\\
         \bottomrule
    \end{tabular}

   \label{tab:minimal_graphs}
    \end{table}

\subsection{Synthetic Data}

We then analyzed synthetic graphs with $n=20$ nodes\footnote{We also tested other values of $n$, obtaining comparable results.} generated by randomly selecting a fixed percentage of edges from the complete graph (named density), discarding disconnected graphs. We considered density values ranging from 5$\%$ to $70\%$: for each value, 1000 graphs were sampled, for a total of 8000 samples. Both our method and \cite{ArrigoniPajdlaAl23} were applied to each graph and they always gave the same output.
Results are collected in Table \ref{tab:synth_results}: as expected, when the percentage of edges decreases, the graphs are more likely to be unsolvable and the number of components increases. 
The examples  shown in Figure \ref{fig:inclusion_components} are taken from this experiment.

\begin{table}[ht]
\caption{Analysis on 1000 random graphs with 20 nodes and varying density.  
Column ``\#fin\_solv'' reports the number of graphs that passed the finite solvability test.
The last column reports the [min max] number of components. }
        \centering
       \begin{tabular}{@{}r r c @{}} 
\toprule
 \%density  & \#fin\_solv  & \phantom{a} \#comp. \phantom{aa}\\
\midrule
5 & 0 &  [10, 25] \\
10 & 2  & [1,  27]\\
20 & 253  & [1, 24] \\
30 & 826  & [1, 5] \\
40 & 977  & [1, 3] \\
50 & 999  & [1, 2] \\
60 & 1000  & [1, 1] \\
70 & 1000  & [1, 1]\\
\bottomrule
\end{tabular}  

         \label{tab:synth_results}
    \end{table}

\begin{table*}[!t]
    \centering
\caption{Results of our experiments on real SfM datasets \cite{WilsonSnavely14,CrandallOwensAl11,OlssonEnqvist11}. \\ ``Time'' is the total time (in seconds) spent building the solvability matrix (Arrigoni et al.\ \cite{ArrigoniPajdlaAl23}) or the Jacobian matrix (our approach), testing for finite solvability, and computing the  components (only on non-solvable cases).  }
\medskip
{
\begin{tabular}{@{}lrrrcrcr@{}} 
\toprule
\multicolumn{4}{c}{Dataset} & \multicolumn{2}{c}{Arrigoni et al. \cite{ArrigoniPajdlaAl23} } & \multicolumn{2}{c}{Our Method} \\
 \cmidrule(lr){1-4} \cmidrule(lr){5-6} \cmidrule(lr){7-8}
Name & \#nodes & \%density  & \#edges & \ \#comp. & Time & \ \#comp. & Time \\ 
\midrule
\rowcolor{parula5!20}Gustav Vasa & 18 & 72 & 110 & 1 & 0.10 & 1 & 0.37 \\ 
\rowcolor{parula5!20}Dino 319 & 36 & 37 & 230 & 1 & 0.07 & 1 & 0.15 \\ 
\rowcolor{parula5!20}Dino 4983 & 36 & 37 & 231 & 1 & 0.06 & 1 & 0.06 \\ 
\rowcolor{parula5!20}Folke Filbyter & 40 & 32 & 250 & 1 & 0.06 & 1 & 0.05 \\ 
\rowcolor{parula5!20}Jonas Ahls & 40 & 41 & 321 & 1 & 0.09 & 1 & 0.20 \\ 
\rowcolor{parula5!20}Park Gate & 34 & 94 & 529 & 1 & 0.13 & 1 & 0.12 \\ 
\rowcolor{parula5!20}Toronto University & 77 & 33 & 974 & 1 & 0.30 & 1 & 0.18 \\ 
\rowcolor{parula5!20}Sphinx & 70 & 55 & 1330 & 1 & 0.53 & 1 & 0.23 \\ 
\rowcolor{parula5!20}Cherub & 65 & 64 & 1332 & 1 & 0.53 & 1 & 0.24 \\ 
\rowcolor{parula5!20}Tsar Nikolai I & 98 & 52 & 2486 & 1 & 1.32 & 1 & 0.46 \\ 
\rowcolor{parula4!20}Skansen Kronan & 131 & 88 & 7490 & 1 & 10.08 & 1 & 2.07 \\ 
\rowcolor{parula4!20}Alcatraz Courtyard & 133 & 92 & 8058 & 1 & 11.59 & 1 & 2.30 \\ 
\rowcolor{parula4!20}Buddah Tooth & 162 & 73 & 9546 & 1 & 15.47 & 1 & 2.93 \\ 
\rowcolor{parula4!20}Pumpkin & 195 & 65 & 12276 & 1 & 26.16 & 1 & 4.63 \\ 
\rowcolor{parula4!20}Ellis Island & 240 & 71 & 20290 & 1 & 74.22 & 1 & 11.54 \\ 
\rowcolor{parula4!20}NYC Library & 358 & 32 & 20662 & 1 & 74.44 & 1 & 12.10 \\ 
\rowcolor{parula4!20}Madrid Metropolis & 370 & 35 & 23755 & 1 & 99.45 & 1 & 15.25 \\ 
\rowcolor{parula4!20}Tower of London & 489 & 20 & 23844 & 4 & 103.83 & 4 & 19.53 \\ 
\rowcolor{parula4!20}Piazza del Popolo & 345 & 42 & 24701 & 4 & 108.57 & 4 & 18.35 \\ 
\rowcolor{parula4!20}Union Square & 853 & 7 & 25478 & 4 & 124.94 & 4 & 27.64 \\ 
\rowcolor{parula4!20}Yorkminster & 448 & 28 & 27719 & 1 & 126.89 & 1 & 19.31 \\ 
\rowcolor{parula4!20}Gendarmenmarkt & 722 & 18 & 48124 & 4 & 411.67 & 4 & 69.81 \\ 
\rowcolor{parula4!20}Montreal N. Dame & 467 & 48 & 52417 & 1 & 462.81 & 1 & 65.53 \\ 
\rowcolor{parula4!20}Roman Forum & 1102 & 12 & 70153 & 4 & 913.57 & 4 & 158.16 \\ 
\rowcolor{parula2!20} Alamo & 606 & 53 & 97184 & 1 & 2335.94 & 1 & 222.35 \\ 
\rowcolor{parula2!20} Vienna Cathedral & 898 & 26 & 103530 & 1 & 2565.31 & 1 & 253.84 \\ 
\rowcolor{parula2!20} Notre Dame & 553 & 68 & 103932 & 1 & 2631.00 & 1 & 249.12 \\ 
\rowcolor{parula2!20} Arts Quad & 5460 & 1 & 221929 & 1 & 10979.35 & 1 & 898.56 \\ 
\rowcolor{parula2!20} Piccadilly & 2446 & 11 & 319195 & 1 & 25889.20 & 1 & 2361.20 \\ 
\bottomrule
\end{tabular}  
}
\label{tab:real_results}
\end{table*}

\subsection{Real Data}
As done in \cite{ArrigoniPajdlaAl23}, we consider real viewing graphs taken from popular structure-from-motion datasets: the Cornell Arts Quad dataset~\cite{CrandallOwensAl11}, the 1DSfM dataset \cite{WilsonSnavely14}, and image sequences from \cite{OlssonEnqvist11}. Some statistics about these graphs are reported in Table \ref{tab:real_results}, namely: the number of nodes; the number of edges; the density (i.e., the percentage of available edges with respect to the complete graph). 
Table~\ref{tab:real_results} also reports the outcome of this experiment: the number of components and the execution times of the competing methods.
Note that $\#$components=1 is equivalent to say that the graph is finite solvable. Information on the number of rows/columns of the matrix used by Arrigoni et al.~\cite{ArrigoniPajdlaAl23} and the one from our formulation \textcolor{black} is given in Figure~\ref{fig:size_real}, whereas explicit formulas are given in Section \ref{sec_size_matrices}.

\begin{figure}[ht]
\centerline{
\includegraphics[width=0.5\columnwidth]{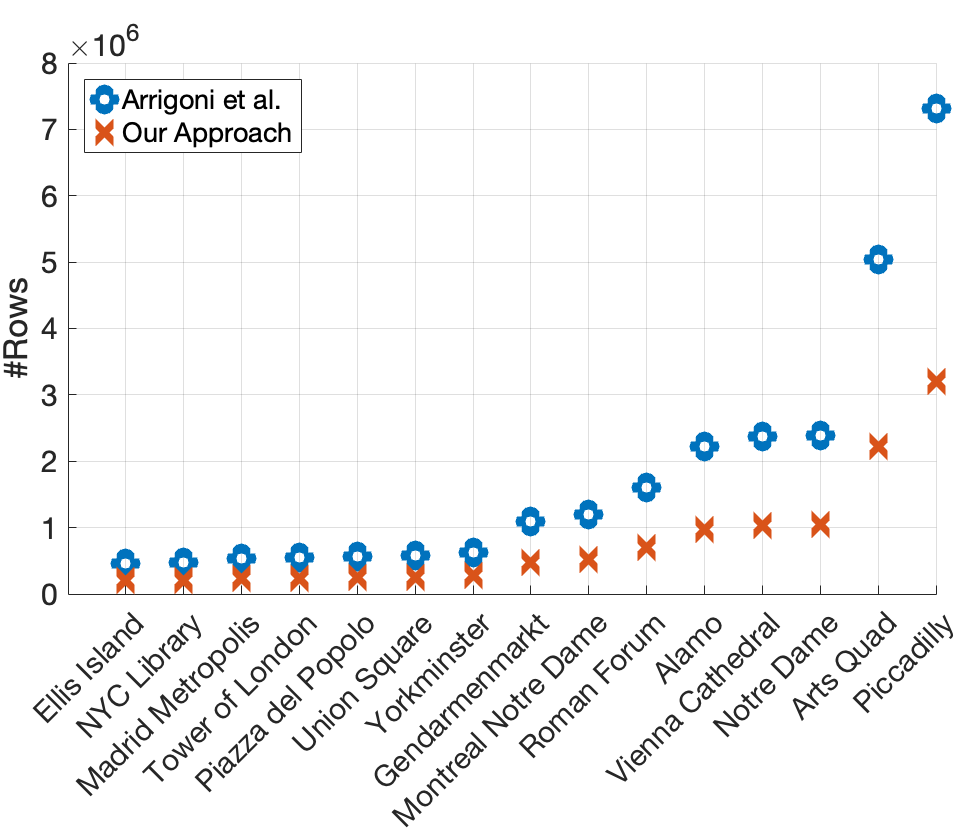}
\includegraphics[width=0.5\columnwidth]{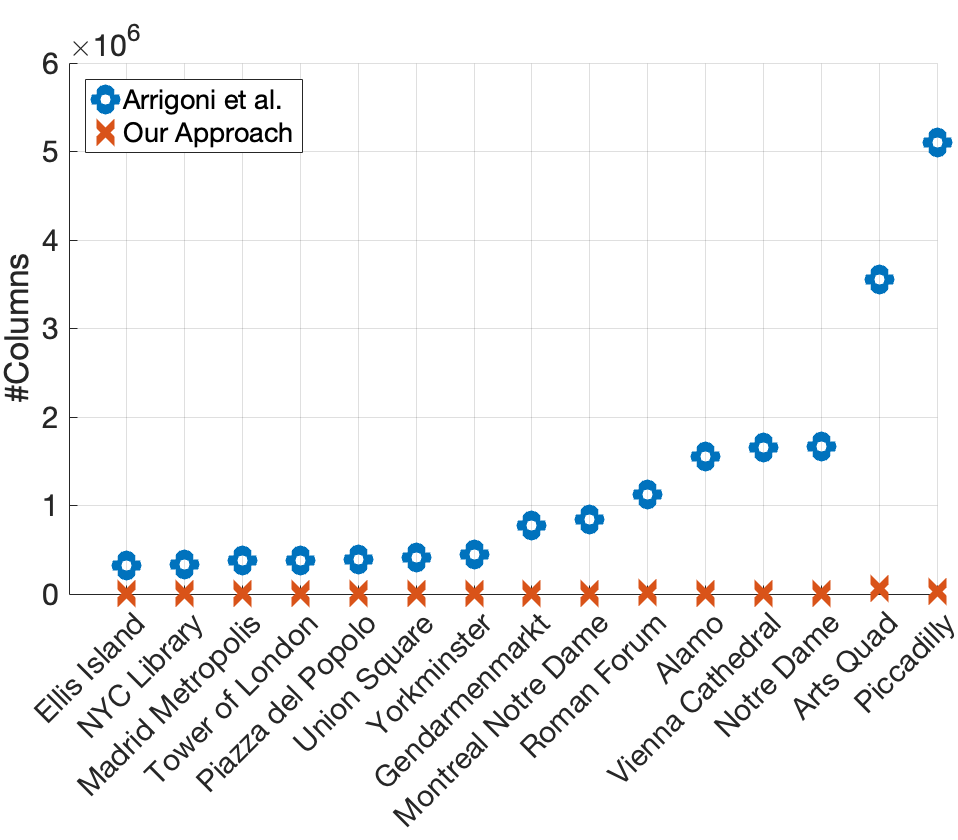}}
    \caption{ Number of rows and columns of the matrix used by Arrigoni et al.~\cite{ArrigoniPajdlaAl23} and the one from our formulation on large-scale SfM datasets \cite{WilsonSnavely14,CrandallOwensAl11}. Observe that the difference in the number of columns even surpasses one order of magnitude.
    }
    \label{fig:size_real}
\end{figure}

Results show that there are only five unsolvable cases among the analyzed graphs, all exhibiting four components, in agreement with previous work. One example is visualized in Figure \ref{fig:components_TOL}. Our method and the one by Arrigoni et al.~\cite{ArrigoniPajdlaAl23} always gave the same output on all the graphs, as expected. 
Table \ref{tab:real_results} also shows that our approach is significantly faster than the state of the art, underlying the advantage of a node-based approach with respect to an edge-based one. Indeed, the matrix employed by our formulation is significantly smaller than the one used by the authors of \cite{ArrigoniPajdlaAl23} -- this can also be seen in Figure \ref{fig:size_real} and Table \ref{tab:size_comparison}.
In particular, our direct formulation takes less than 10\% of the total running time of \cite{ArrigoniPajdlaAl23} on the largest examples (from ``Alamo'' to ``Piccadilly''). This figure becomes 20\% for medium size datasets (from ``Skansen Kronan'' to ``Roman Forum''). For the smallest ones the running time is less than a second and the comparison becomes meaningless.

\begin{figure}[ht]
\centering
\includegraphics[width=0.95\columnwidth]{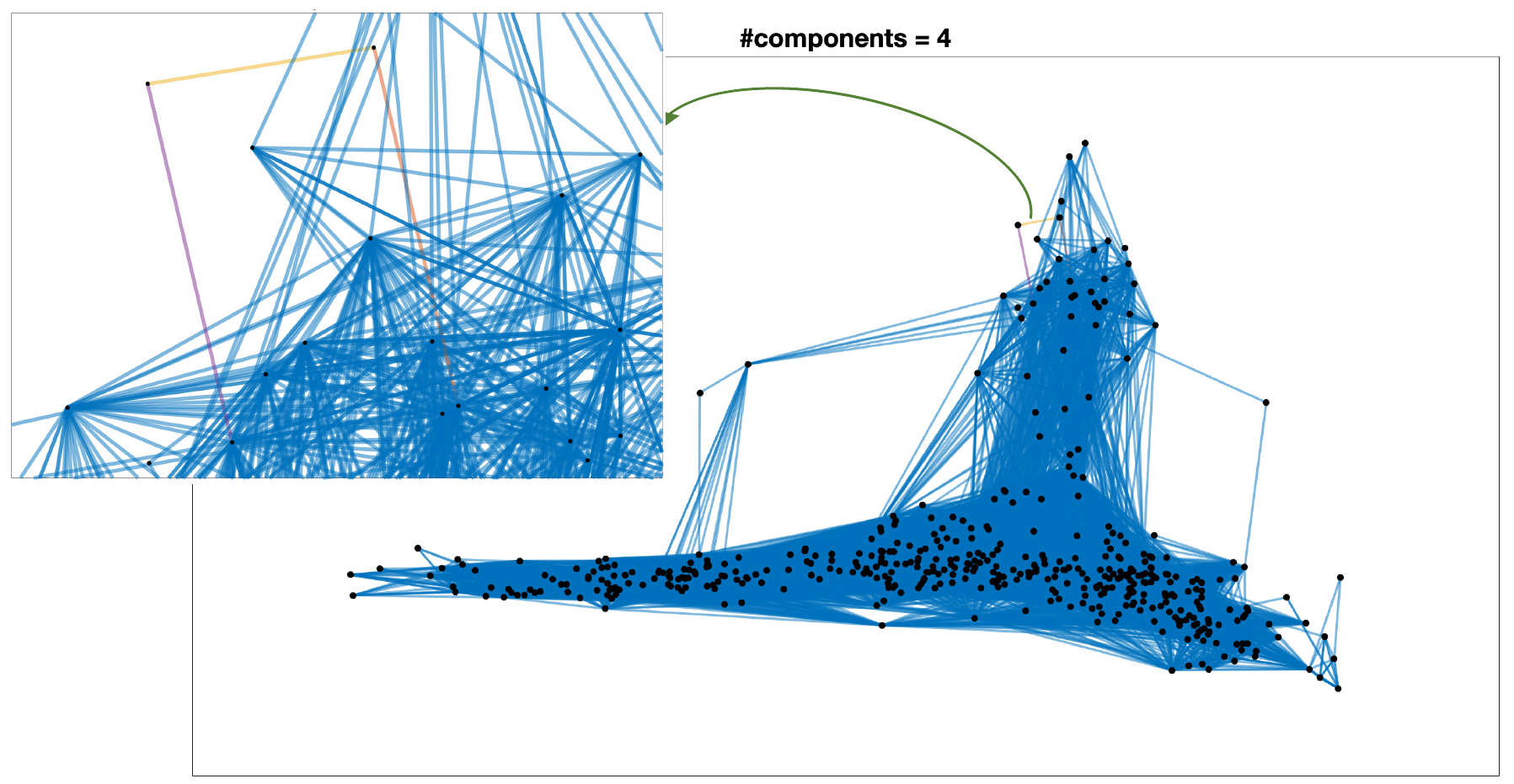}
    \caption{
    Viewing graph of the Tower of London dataset \cite{WilsonSnavely14} and maximal components (color-coded). Edges in the largest components are depicted in blue. A zoom is drawn to better visualize the non-solvable part, comprising three edges which resemble the square topology.
    }
    \label{fig:components_TOL}
\end{figure}


\subsection{Number of Rows/Columns for Different Formulations}
\label{sec_size_matrices}

Here we summarize the comparison with \cite{ArrigoniPajdlaAl23} and \cite{TragerOssermanAl18} in terms of the size of the respective matrices. Recall that $n=|V|$ and $m=|E|$ in a graph $G$ with $V$ vertices and $E$ nodes.

\smallskip 

\textbf{Trager et al.~\cite{TragerOssermanAl18}.}
The solvability matrix of \cite{TragerOssermanAl18} is made of blocks, where each block comprises 20 equations. The number of blocks per node is $d_i (d_i-1) /2 $, where $d_i$ denotes the degree of node $i$ (see Table 2 in \cite{ArrigoniPajdlaAl23}). 
By summing over all the nodes in the graph, a formula is obtained for the number of rows of the solvability matrix (or, equivalently, the number of equations):
\begin{equation}
 e_1 
 =  10 \sum_{i=1}^n ({d_i^2} - d_i) + 15 + m = 10 \sum_{i=1}^n {d_i^2} - 19 m + 15
 \label{eq_unknowns_trager}
\end{equation}
where $15 + m$ accounts for the additional equations introduced to remove the ambiguities and $\sum_{i=1}^n d_i = 2m$ due to the \emph{degree sum formula} \cite{Harary72}. 
Exploiting the Cauchy-Schwarz inequality 
we obtain:
\begin{equation}
    \sum_{i=1}^n d_i^2 \geq \frac{1}{n} \left( \sum_{i=1}^n d_i \right)^2
\end{equation}
hence, using again the degree sum formula, we get the following lower bound for $e_1$: 
\begin{equation}
 e_1 \geq    \frac{10}{n} \left( 2m  \right)^2 - 19m + 15 = 
 40 \frac{m^2}{n} -19m + 15.
\label{eq_5}
\end{equation}
Hence the number of rows grows asymptotically (at least) as $O(n^3)$ for a dense graph (where $ m \approx O(n^2) $) and it grows as $O(n)$ for a sparse one (where $ m \approx O(n) $).

\medskip

\textbf{Arrigoni et al.~\cite{ArrigoniPajdlaAl23}.} The reduced solvability matrix used by \cite{ArrigoniPajdlaAl23} is made of blocks of 11 equations. The number of blocks per node is $d_i-1$ (therefore it scales linearly in the degree of node $i$ whereas in \cite{TragerOssermanAl18} the growth is quadratic).
Hence the total number of rows (i.e., equations) is given by:
\begin{equation}
\begin{aligned}
  e_2 &=   11 \sum_{i=1}^n (d_i - 1 ) + 15 + m = 11 \sum_{i=1}^n d_i -  11 n + 15 + m \\
  &=  23 m -  11 n + 15   .  
\end{aligned}
 \label{eq_6}
\end{equation}
The above formula implies that the number of rows of the reduced solvability matrix grows asymptotically  as $O(n^2)$ for a dense graph and $O(n)$ for a sparse one.
Away from the limit case of a perfectly sparse graph with $ m = O(n) $, there is an advantage of this formulation with respect to \cite{TragerOssermanAl18}. 
In concrete terms, it is enough that $m > n$ to ensure that  $e_2 \leq e_1$: indeed, after proper simplifications, \eqref{eq_5} $ \geq $ \eqref{eq_6} becomes
$40 m^2 + 11 n^2  \geq 42 nm  >  42 n^2$,
which reduces to $ 40 m^2 >  31 n^2$, which is always satisfied under the hypothesis $m >  n$.  
The number of columns (i.e., variables) is the same for \cite{TragerOssermanAl18} and \cite{ArrigoniPajdlaAl23}, and it is given by 
\begin{equation}
  v_1=v_2=16m.  
\end{equation}
We refer the reader to \cite{ArrigoniPajdlaAl23} for additional information on the performance of \cite{TragerOssermanAl18} on real-world datasets.

\medskip

\textbf{Our Formulation.}
As explained in Section \ref{sec_method}, our polynomial system employs a total of
\begin{equation}
    \begin{gathered}
  e_3=10m+n+15 \ \text{ equations} \\
  v_3=12n \ \text{ unknowns.}
    \end{gathered}
\end{equation}
Hence, the number of rows of our Jacobian matrix grows asymptotically as $O(n^2)$ for a dense graph and $O(n)$ for a sparse one.
In concrete terms, however, $e_3 \leq e_2$ as soon as $m \ge \frac{12}{13} n \approx n$, which is typically satisfied.

\medskip 

A summary is reported in Table \ref{tab:size_comparison}. Note that our formulation is the only one where the number of columns scales with the number of nodes (instead of edges) in the graph, as ours is the first node-based method. Observe also that practical datasets are far from the sparse graph approximation, as the number of edges is much larger than the number of nodes.

\begin{table}[htbp]
    \caption{Number of equations/unknowns for the three formulations. The row counts are in decreasing order for typical graphs.
    }
    \label{tab:size_comparison}
    \centering
   \begin{tabular}{@{}llc@{}}
\toprule
Method   &   \#rows    &   \#columns  \\ \midrule
Trager et al.~\cite{TragerOssermanAl18}   & $ \geq 40 {m^2}/{n} -19m + 15 $ & $ 16m $ \\ 
Arrigoni et al.~\cite{ArrigoniPajdlaAl23} & $ 23 m - 11 n + 15 $  & $ 16m$ \\
Ours      &  $ 10m + n + 15 $      & $  12n$     \\ \bottomrule
\end{tabular}
\end{table}


\section{Conclusion}
\label{sec_conclusion}

This paper underscored the viewing graph as a powerful representation of uncalibrated cameras and their geometric relationships. The solvability of the graph corresponds to the existence of a unique set of cameras, up to a single projective transformation, that conforms to the given fundamental matrices. Our focus was on the relaxed notion of finite solvability, which considers the finiteness of solutions rather than strict uniqueness. This approach is computationally tractable and enables the analysis of large graphs derived from structure-from-motion datasets.

We presented a novel formulation of the problem that provides a more direct approach than previous literature -- based on a formula that explicitly establishes links between pairs of cameras through their fundamental matrices, as suggested by the definition of solvability. 
Building upon this, we developed an algorithm designed to test finite solvability and extract components of unsolvable cases, surpassing the efficiency of previous methods. 
The core methodology is mathematically sound and extremely simple, as it only requires computing the derivatives of polynomial equations with respect to their unknowns, and checking the rank of the resulting Jacobian matrix. 
Although the Jacobian check, by definition, applies only to a neighborhood of a particular solution, in our case, this local information extends globally due to the special structure of the problem. We formally established this result—originally conjectured in our preliminary study \cite{ArrigoniFusielloAl24}—using tools from Algebraic Geometry.

The concept of finite solvability, while valuable, represents only a partial step toward a computationally efficient characterization of viewing graph solvability. Its inherent limitation lies in asserting the existence of a finite number of solutions rather than guaranteeing a unique one. The challenge of efficiently verifying uniqueness in large structure-from-motion graphs remains an open question. We hope that our results will inspire further research in this intriguing direction.

%
%

\section*{Acknowledgements}
Federica Arrigoni was supported by PNRR-PE-AI FAIR project funded by the NextGeneration EU program. Tomas Pajdla was supported by the OPJAK CZ.02.01.01/00/22 008/0004590 Roboprox Project.
Kathlén Kohn was supported by the Wallenberg AI, Autonomous Systems and Software Program (WASP) funded by the Knut and Alice Wallenberg Foundation.

\printbibliography

\end{document}